\documentclass[10pt, conference, compsocconf, preprint]{IEEEtran}
\usepackage{booktabs} % For formal tables
\usepackage{balance}
\usepackage{graphicx}
\usepackage[]{algorithm2e}
\usepackage{placeins}
\usepackage{amsmath}

\ifCLASSINFOpdf
\else
\fi

\begin{document}
%
% paper title
% can use linebreaks \\ within to get better formatting as desired
\title{A Data Analytics Framework for Aggregate Data Analysis}

% author names and affiliations
% use a multiple column layout for up to two different
% affiliations

\author{\IEEEauthorblockN{Sanket Tavarageri~~~~~~~~Nag Mani}
\IEEEauthorblockA{Department of Computer Engineering\\
San Jos\'{e} State University\\
Email: \{sanket.tavarageri, nag.mani\}@sjsu.edu}
\and
\IEEEauthorblockN{Anand Ramasubramanian~~~~~~~~Jaskiran Kalsi}
\IEEEauthorblockA{Department of Biomedical Engineering\\
San Jos\'{e} State University\\
Email: \{anand.ramasubramanian,jaskiran.kalsi\}@sjsu.edu}
}

% conference papers do not typically use \thanks and this command
% is locked out in conference mode. If really needed, such as for
% the acknowledgment of grants, issue a \IEEEoverridecommandlockouts
% after \documentclass

% for over three affiliations, or if they all won't fit within the width
% of the page, use this alternative format:
% 
%\author{\IEEEauthorblockN{Michael Shell\IEEEauthorrefmark{1},
%Homer Simpson\IEEEauthorrefmark{2},
%James Kirk\IEEEauthorrefmark{3}, 
%Montgomery Scott\IEEEauthorrefmark{3} and
%Eldon Tyrell\IEEEauthorrefmark{4}}
%\IEEEauthorblockA{\IEEEauthorrefmark{1}School of Electrical and Computer Engineering\\
%Georgia Institute of Technology,
%Atlanta, Georgia 30332--0250\\ Email: see http://www.michaelshell.org/contact.html}
%\IEEEauthorblockA{\IEEEauthorrefmark{2}Twentieth Century Fox, Springfield, USA\\
%Email: homer@thesimpsons.com}
%\IEEEauthorblockA{\IEEEauthorrefmark{3}Starfleet Academy, San Francisco, California 96678-2391\\
%Telephone: (800) 555--1212, Fax: (888) 555--1212}
%\IEEEauthorblockA{\IEEEauthorrefmark{4}Tyrell Inc., 123 Replicant Street, Los Angeles, California 90210--4321}}

% use for special paper notices
%\IEEEspecialpapernotice{(Invited Paper)}

% make the title area
\maketitle
\thispagestyle{empty}
\pagestyle{plain}

\newtheorem{theorem}{Theorem}

\begin{abstract}
In many contexts, we have access to aggregate data, but individual level data
is unavailable.
For example,
%voting data in elections at the precinct level may be
%available, however owing the secret ballot nature of elections in many countries
%including the U.S., we cannot know how each individual within the precinct voted.
medical studies sometimes report only aggregate statistics about disease
prevalence because of privacy concerns.
%Yet another example would be sales in a supermarket: data on what product is sold
%in what quantities is available, but tracing the sales to individuals may not always
%be possible.
    Even so, many a time it is desirable, and in fact could be necessary
    to infer individual level characteristics from aggregate data.
    For instance, other researchers who want to perform more detailed
analysis of disease characteristics would require individual level data.
%any action based on the U.S. Voting Rights Act requires
%establishing that minority
%    groups vote differently than other voters and it necessitates aggregate
%data analysis of
%precinct-level
%    voting data.
    Similar challenges arise in other fields too including politics,
and marketing.
%    Detailed analysis and pattern recognition of disease prevalence necessitates
%    coming up with patent-level data.
%    Similarly, running effective advertising campaigns for grocery items mandates
%    that the buying patterns of different segments of the customer base be known.

In this paper, we present an end-to-end pipeline for processing of aggregate
    data to derive individual level statistics, and then using the inferred data
    to train machine learning models to answer questions of interest.
We describe a novel algorithm for reconstructing  fine-grained data
    from summary statistics.
This step will create multiple candidate datasets which will form the input
    to the machine learning models.
    The advantage of the highly parallel architecture we propose is that
    uncertainty in the generated fine-grained data will be compensated by the use
    of multiple candidate fine-grained datasets.
    Consequently, the answers derived from the machine learning models
    will be more valid and usable.
    We validate our approach using data from a challenging medical problem
called Acute
    Traumatic Coagulopathy.

\end{abstract}

%\begin{IEEEkeywords}
%component; formatting; style; styling;

%\end{IEEEkeywords}

% For peer review papers, you can put extra information on the cover
% page as needed:
% \ifCLASSOPTIONpeerreview
% \begin{center} \bfseries EDICS Category: 3-BBND \end{center}
% \fi
%
% For peerreview papers, this IEEEtran command inserts a page break and
% creates the second title. It will be ignored for other modes.
\IEEEpeerreviewmaketitle

\section{Introduction}
Reconstructing individual behavior from aggregate data is termed \emph{ecological inference}
\cite{king2013solution}.
The necessity for ecological inference occurs because
1) the underlying data that gave rise to the aggregate statistics is unavailable and,
2) the analysis that we intend to carry out requires individual level data.
Examples for the need for this kind of analysis abound in various fields.
Below, we list a few scenarios.

\begin{itemize}
    \item Medical studies sometimes report only aggregate statistics about
    prevalence of a disease because
    of privacy concerns.
    Even if data is de-identified before it is put in the public domain, it is
    susceptible to re-identification attacks \cite{el2011systematic}.
    Therefore, medical professionals often choose to only publish aggregate
    statistics out of abundance of caution.
    Yet other researchers working to understand the disease better may wish to
    regenerate
    the original data
    for detailed analysis.
    E.g., for pattern recognition, and for building machine learning models to
    predict outcomes.
%   individual patient records have to be used.
    In this paper, we use aggregate data about a condition called \emph{Acute
    Traumatic Coagulopathy} (ATC) as the use case to demonstrate the algorithms, and
    the system developed in this work.
    Specifically, we use the \emph{Odds Ratios (ORs)} of the known factors causing
    ATC published in a medical journal and reconstruct patient data.
    We use the regenerated patient data to train machine learning models to predict
    mortality.
\item    The voting data of elections is generally available at the precinct
    level.
However as ballots are cast in secret, data on how each individual voted
cannot be known.
Politicians and political scientists are often interested in knowing how
different demographic groups voted.
To come up with a reasonably valid answer to this question, one could
couple the voting data with the Census data concerning the precinct and
reconstruct the individual voting behavior.
Enforcement of certain laws may require having individual level
voting data at hand.
For instance, the Voting Rights Act prohibits voting discrimination based on
race, color, or language.
The plaintiffs challenging any alleged discrimination have to first demonstrate
that minority groups vote differently than majority groups,
which can be done via ecological inference.
\item Sales in a supermarket: data on what product is sold
in what quantities in a supermarket is available, but tracing the sales to
individuals may not always
be possible.
Running effective advertising campaigns for grocery items mandates
that the buying patterns of different segments of the customer base be known
so that those segments can be reached via relevant advertisements.
\end{itemize}

%\begin{figure*}[ht]
%    \centering
%    \includegraphics[scale=0.45]{figures/coagulation.png}
%    \caption{Blood clotting cascade summary. Source: https://en.wikipedia
%    .org/wiki/File:Coagulation\_full.svg}
%    \label{figure:clotting_cascade}
%\end{figure*}

In this paper, we present the design and implementation of a highly scalable
system for analyzing aggregate data.
We develop a novel algorithm to reconstruct individual level attributes
from
summary statistics.
Because this is a probabilistic method, to increase the confidence in the
analysis, the reconstruction algorithm outputs multiple candidate datasets.
Each of the candidate datasets is then used to train a machine learning model
to predict a quantity of interest.
Thus, at the end of the training step, we have an \emph{ensemble} method to
predict the outcome.

We run the outlined pipeline in the context of ATC data.
Additionally, to validate the approach we synthesize various patient datasets
that have a
range of ORs.
The synthetic datasets serve as the ground truth for validation:
we run the entire pipeline -- compute summary statistics, reconstruct several
candidate datasets, train the machine learning models, and obtain prediction
results.
The predicted outcomes are compared with the ground truth.
We show that the error rates are low indicating that the reconstruction
algorithms, and the machine learning models are effective in understanding
the underlying processes.

The contributions of the paper are as follows.
\begin{itemize}
    \item We present a novel data reconstruction algorithm to regenerate
    individual level
    data from aggregate data given odds ratios.
    \item A scalable data pipeline architecture is developed to train machine
     learning models with the multiple reconstructed datasets.
    \item We present the results of an extensive experimental evaluation
    validating the proposed approach.
\end{itemize}

The rest of the paper is organized as follows.
We introduce the Acute Traumatic Coagulopathy condition in Section
\ref{section:act}.
The aggregate data available for ATC will be used as a running example to
describe the algorithms, and the techniques developed in the paper.
Section \ref{section:pipeline} develops the data reconstruction algorithm.
In that section, we delineate the system architecture that uses
the regenerated data to train machine learning (ML) models in parallel.
Application of the developed system in the context of ATC as well as
the experiments performed to assess the efficacy
of the system are described in Section \ref{section:case_study}.
The related work is discussed in Section \ref{section:related}.
Section \ref{section:conclusion} concludes the paper with the key findings of
the work.

\section{Acute Traumatic Coagulopathy}
\label{section:act}

Acute Traumatic Coagulopathy (ATC) \cite{dirkmann2008hypothermia,hess2008coagulopathy} is a
condition characterized by
prolonged and/or excessive bleeding immediately following a traumatic injury.
Despite the many recent advances in trauma care, failure to stop bleeding (hemostasis)
following hemorrhage and shock remains the leading cause of death among
children and adolescents \cite{centers2012vital}.
This condition may present as early as 30 minutes after trauma prior to
intervention,
and this period is particularly critical in determining the mortality rates.
It is associated with higher injury severity, coagulation abnormalities, and increased blood transfusions.
The unfortunate consequences are poor clinical outcomes, and high mortality rates in trauma patients.

The underlying biochemical mechanisms that lead to ATC are not definitively known,
and it results from a wide range of symptoms and phenotypes seen in the
patients \cite{meledeo2017acute}.
A lot of the complexity stems from the fact that there are two distinct
pathways -- intrinsic, and extrinsic -- that cause the blood to clot.
%Figure \ref{figure:clotting_cascade} delineates the two pathways -- intrinsic
%and extrinsic.
%The Roman letters denote various proteins.
If any of the biochemical reactions in the cascade breaks down or is impaired,
that causes
insufficient coagulation and may lead to ATC.
Since ATC is failure of the coagulation system, the laboratory test for identifying ATC
has historically been prolonged prothrombin time (PT) or hypocoagulable condition.
However, this is known to be not true: there are patients arriving with
shortened PT
following trauma, particularly after burn injury and yet are hypocoagulable.
Newer measures such as injury severity score, partial thromboplastin time (PTT),
degree of fibrinolysis, depletion of coagulation factors and inhibitors,
and general failure of blood system have all been identified as primary
indicators of ATC \cite{macleod2003early}.

However, there are inherent discrepancies in the diagnostic tests due to timely sample collection,
quality and availability of assays, lack of baseline pre-injury measurements, inter-individual variability,
and the multivariate nature of coagulopathy itself.
These issues have made the conventional, reductionist approach to understanding
and treatment of ATC a failure, warranting a holistic approach instead.
The long-term goal of ours is to develop machine learning
approaches in conjunction with clinical assays to understand the physiological
mechanisms of ATC as well as to predict the phenotype and treatment outcomes dependably.
The objective in this paper is to develop computational models that can
classify ATC phenotypes
as a function of various known ATC indicators.
Our central hypothesis is that we can use machine learning technology to model
the complex interplay between various hematological and physiological parameters
and predict the chances of ATC accurately.
The rationale for the current research is that the discerned mathematical
relationship
between the various indicators and clinical outcomes will naturally lead to
hypotheses that can subsequently be tested experimentally .

Figure \ref{figure:ORs} shows the impact of abnormal coagulation
parameters in terms of odds ratios reported in a medical journal.

\begin{figure}[h]
    \centering
    \includegraphics[scale=0.50]{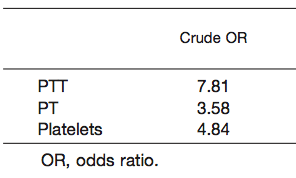}
    \caption{Example of aggregate data on ATC. Source: MacLeod et al.
    \cite{macleod2003early} }
    \label{figure:ORs}
\end{figure}

The figure indicates the odds ratio (OR) of 3.58 for death with an abnormal PT.
Table \ref{table:PT_OR}
enumerates the
number of
patients that had an
abnormal PT value, and the number of people who died.
The \emph{odds} of dying given an abnormal PT value is $\frac{579}{2,415} = 0
.240$,
while the \emph{odds} of dying given a normal PT value is $\frac{489}{7,307} = 0
.067$.
The odds ratio (OR) is the ratio of the two quantities: $\frac{0.240}{0
.067} = 3.58$.
When the OR is greater than 1, having an abnormal PT is considered to be
associated with dying.
Higher the OR, greater is the association between abnormal PT and dying.

\begin{table}[h]
    \caption{PT odds ratio}
    \label{table:PT_OR}
    \begin{center}
        \begin{tabular}{ | l | r | r| }
            \hline
            & Dead & Survived \\ \hline
            Abnormal PT & 579 & 2,415 \\ \hline
            Normal PT & 489 & 7,307 \\
            \hline
        \end{tabular}
    \end{center}
\end{table}

The OR for the other two parameters namely, PTT and Platelets have a similar
meaning.
The \emph{ecological inference} problem is to generate
individual patient level records for $n$ patients such that the ORs would be
as given in Figure \ref{figure:ORs}.
Let us denote a normal PT/PTT/Platelets value with 1, and an abnormal value
with 0.
The ``Dead?'' column will be populated with a 1 if the patient survived, and a
0 if the patient succumbed.
Then, the problem becomes one that of filling Table \ref{table:candidate_dataset}
with 0s, and 1s in such a way that when we compute the ORs on this Table,
they would match with the ones in Figure \ref{figure:ORs}.

\begin{table}[h]
    \caption{Plausible patient data}
    \label{table:candidate_dataset}
    \begin{center}
        \begin{tabular}{ | r | c | c| c | c| }
            \hline
            Id & PTT & PT & Platelet & Dead? \\ \hline \hline
            1 & \large{?} & \large{?} & \large{?} & \large{?} \\
            2 & \large{?} & \large{?} & \large{?} & \large{?} \\
            3 & \large{?} & \large{?} & \large{?} & \large{?} \\
            \vdots & \vdots & \vdots & \vdots & \vdots \\
            $n$ & \large{?} & \large{?} & \large{?} & \large{?} \\
            \hline
        \end{tabular}
    \end{center}
\end{table}

The challenges to address include:
\begin{itemize}
    \item    How to efficiently explore the huge space of possible tables?
    We notice that for the four unknown columns case in Table
    \ref{table:candidate_dataset} where each cell assumes either a 0 or a 1, we
    have a total of $2^{4n}$ unique probable tables.
    In general, when the number of columns is $m$, and each column
    corresponds to a $k$ valued attribute, the number of combinations is
    $k^{mn}$.

    \item How do we exploit the given summary statistics such as odds ratios to
    narrow down the search space, and speed up data reconstruction?

    \item How to quantify the feasibility of the generated table(s)?
    Given the fact that a large number of feasible tables could have given rise
    to a given set of aggregate statistics, is there a way to
    empirically assess the validity of the reconstructed table(s)?
\end{itemize}

We address these challenges next.

\section{Parallel Data Pipeline Architecture}
\label{section:pipeline}

\subsection{Ecological Inference Algorithm}
\label{section:inference}

We illustrate the solution approach on the ATC data first, and then
generalize the solution.
To generate one candidate individual level dataset, we fix the outcome column
-- ``Dead?''
first, and then populate the other feature columns -- PT, PTT, Platelet.
From Table \ref{table:PT_OR} it is seen that the total number of patients in
this study is $579 + 2,415 + 489 + 7,307 = 10,790$.
Of the $10,790$ patients $579 + 489 = 1,068$ died.
Therefore, $1,068$ Ids among $10,790$ Ids are randomly selected, and are
marked as 0 to indicate that patients with those Ids died.
The rest are marked 1.

We subsequently demarcate patients who have abnormal PT values.
A total of $579 + 2,415 = 2,994$ individuals had an abnormal PT.
We also know that $579$ of them died.
Since we already have identified those that have died and in particular $1,
068$ of them, we consider this set and randomly select $579$ Ids from this
set, and assign them an abnormal PT (0).
The remaining Ids in the \emph{dead} set are given a normal PT score (1).
Similarly, we have previously determined who survived.
Among the survived set, randomly selected $2,415$ persons will receive an
abnormal PT, and the rest will get a normal PT value.
We repeat this exercise for the PTT, and Platelet columns.

In this problem, we are attempting to model mortality as a function of PTT,
PT, and Platelet.
Therefore, ``Dead?'' is the outcome variable, and the others are predictor
variables/features.
The key to being able to construct the table efficiently is to first fix the
outcome
variable and then fill in the predictor variables.

\begin{algorithm}
    \KwData{Features: $x_1, x_2, \dots x_n$, Outcome: $y$ \\
    Odds ratios: $o_1, o_2, \dots o_n$, Occurrence ratios: $p_1, p_2, \dots
    p_n$ \\
    Outcome classes: $c_1, \dots c_k$,
    Class ratios: $r_1, \dots r_k$ \\ Individual observations: $N$ \\
    Seed for random number generator: $s$}
    \KwResult{$N$ individual observations $\mathcal{S}$ with $x_1, x_2, \dots
    x_n, y$
    populated}
    Initialize random number generator with seed $s$ \\
    \For{$i\leftarrow 1$ \KwTo $k$}{
    $\mathcal{S}_i \leftarrow$ Randomly select $r_i  N$ observations \\
    $\mathcal{S}_i[y] \leftarrow$ class label $c_i$
    }

    \For{$j\leftarrow 1$ \KwTo $n$}{
    \For{$i\leftarrow 1$ \KwTo $k$}{
    ${l} \leftarrow$ Solve for four unknowns using four equations involving
    $o_j, p_j, r_i, N$ \\
    $\mathcal{S}_{ji} \leftarrow$ Randomly select
    ${l}$ observations from $\mathcal{S}_{y =
    c_i}$ \\
    $\mathcal{S}_{ji}[x_j] \leftarrow c_{x_{ji}}$
    }
    }

    \caption{Individual data reconstruction from aggregate statistics}
    \label{algo:reconstruction}
\end{algorithm}

Algorithm \ref{algo:reconstruction} presents the steps to compute individual
level data using aggregate statistics.
The algorithm takes as input the features, the outcome variable, odds ratios
for
different features, the fraction of observations that carry positive class
within the features, the outcome classes, fraction of positive and negative
class within the outcome variable.
Additionally, the seed to the random number generator is inputted.
By varying the random number generator seed, we will be able to generate
multiple candidate individual datasets.

The algorithm first populates the outcome variable column $y$.
The proportion of observations that have class label $c_i$ is $r_i$.
Therefore, $r_1 N$ observations among $N$ observations are randomly selected
and are assigned label $c_1$.
Next, the remaining $N - r_1 N$ records are given label $c_2$.

\begin{table}[h]
    \caption{Odds ratio for feature $x_j$}
    \label{table:feature_OR}
    \begin{center}
        \begin{tabular}{ | l | r | r| }
            \hline
            & $y = c_1$ & $y = c_2$ \\ \hline
            $x_j = c_{x_{j1}}$ & $l_1 = ?$ & $l_2 = ?$ \\ \hline
            $x_j = c_{x_{j2}}$ & $l_3 = ?$ & $l_4 = ?$ \\
            \hline
        \end{tabular}
    \end{center}
\end{table}

\begin{align}
    \frac{l_1 . l_4}{l_2 . l_3} & = o_j \\
    \label{eqs:eq1}
    l_1 + l_3 &= r_1 N \\
    l_2 + l_4 &= r_2 N \\
    \frac{l_1 + l_2}{l_3 + l_4} &= p_j
    \label{eqs:eq4}
\end{align}

The different feature values for features $x_1, x_2, \dots x_n$ are
subsequently populated.
Table \ref{table:feature_OR} and equations \ref{eqs:eq1}  through \ref{eqs:eq4}
show the relationship between various parameters.
The goal is to 1) find the values of $l_1$, $l_2$, $l_3$, and $l_4$ such that
the various constraints are met 2) select individual records to assign class
labels $c_{x_{j1}}$ or $c_{x_{j2}}$ for each feature $x_j$.

We have four equations \ref{eqs:eq1} $-$ \ref{eqs:eq4} and four unknowns $l_1 -
l_4$.
We solve for $l_i$s.
$l_1$ individual
records are selected that already have $y = c_1$ and label $c_{x_{j1}}$ is
assigned to them.
The remaining $l_3$ records are given label
$c_{x_{j2}}$.
Among the records that have $y = c_2$, we choose $l_2$ records and set
$c_{x_{j1}}$.
Finally, the unassigned $l_4$ observations will receive label $c_{x_{j2}}$.
If certain inputs to Algorithm \ref{algo:reconstruction} are not known in a
given application say a few of occurrence ratios $p_i$s, but only the
numerical ranges the inputs can assume are provided, then the algorithm selects
values within the range for those inputs.

\subsection{Multiple Candidate Datasets}
\label{section:ml}

The original dataset $\mathcal{H}$, which is hidden/unavailable is summarized
using the aggregate statistics.
In \S\ref{section:inference}, we presented an algorithm that generates one
plausible candidate dataset $\mathcal{S}$ whose characteristics in terms of
aggregate statistics are identical to that of the original dataset
$\mathcal{H}$.
In this section, we will define metrics to quantify the similarity between
$\mathcal{H}$ and $\mathcal{S}$.
We will develop a methodology using Algorithm \ref{algo:reconstruction} to
generate
several
plausible candidate datasets -- $\mathcal{S}_1, \mathcal{S}_2, \dots
\mathcal{S}_n$ so that we can provide strong guarantees on the
similarity between $\mathcal{H}$ and $\mathcal{S}$s.

The central idea is that the plausible datasets are generated in such a way
that any two datasets $\mathcal{S}_i$ and $\mathcal{S}_j$
are sufficiently distinct from each other.
Consequently,  as we increase the number of plausible datasets
generated -- $n$, it will be increasingly likely that a row that appeared in the
original dataset $\mathcal{H}$ will appear in at least one of $\mathcal{S}$s.

\subsubsection{Similarity score}
\label{section:similarity_score}
To compute the similarity score between two datasets $\mathcal{S}_i$ and
$\mathcal{S}_j$, we proceed as follows.

%\begin{enumerate}
%    \item
\textbf{1) Normalization:} The two datasets are normalized using Min-Max
     scaling.
    The minimum, and maximum values of each attribute are calculated -
     $c_{min}$, and $c_{max}$.
    The attribute value $c$ is then subtracted with $c_{min}$ and divided by
the
    difference between minimum and maximum values.
    $$ \frac{c - c_{min}}{c_{max} - c_{min}} $$
As a result, the attribute values after scaling will be between 0 and 1.
The min-max scaling ensures that when we calculate the distance between two
rows (explained below), all attributes contribute equally.

%    \item
\textbf{2) Manhattan distance:} The distance between a row of
     $\mathcal{S}_i$ and that of
    $\mathcal{S}_j$ is defined as the Manhattan distance between them scaled by
the
inverse of number of attributes.
If $\mathcal{F}$ is the set of attributes then the distance between two rows
$r_m$, and $r_n$ is defined as follows.
$$ \text{dist($r_m$, $r_n$)} = \sum_{l \in \mathcal{F}} \frac{|r_m[l] - r_n[l]|}{|\mathcal{F}|} $$

\textbf{3) Average distance:} The average distance among pairs
of rows of the two datasets is computed.
The similarity score is one minus the average distance.
The similarity score ranges between 0 and 1 with 1 indicating that the two
datasets exactly match, while 0 denotes that the two datasets are
very dissimilar.

$$ \text{sim}(\mathcal{S}_i, \mathcal{S}_j) = 1 - \text{avg\_dist}(\mathcal{S}_i, \mathcal{S}_j) $$

\textbf{4) Matching in a Bipartite graph:} We derive a mapping between
     rows of $\mathcal{S}_i$
     and
    $\mathcal{S}_j$
    that minimizes the sum of distances among all pairs of rows.
    Minimizing the sum of distances between paired rows also minimizes the
    average distance between the two datasets.
    We show that this problem is equivalent to the matching problem in complete
    bipartite graphs.
%    \item

While calculating the distance between two datasets - $\mathcal{S}_i$ and
$\mathcal{S}_j$, we pair rows of the two datasets and calculate pairwise row
distances.
It is observed that the exact ordering of the rows does not matter from the
point of view of reconstruction of individual level data from the summary
statistics.
Therefore, we need to find a one-to-one mapping of rows of one dataset into
the other that yields maximum similarity between the two datasets.
We illustrate the scenario with the following example.
Consider the datasets $S_1$ and $S_2$ shown in Table \ref{table:candidate_dataset1}
and Table \ref{table:candidate_dataset2} respectively.

\begin{table}
    \parbox{.45\linewidth}{
    \centering
    \caption{Candidate dataset $\mathcal{S}_1$}
    \label{table:candidate_dataset1}
    \begin{tabular}{ | r | c | c| c | c| }
        \hline
        Id & PTT & PT & Plt. & Dead? \\ \hline \hline
        1 & 1 & 0 & 1 & 1 \\
        2 & 0 & 1 & 0 & 0 \\
        3 & 0 & 0 & 1 & 1 \\
        4 & 0 & 1 & 1 & 1 \\
        \hline
    \end{tabular}
    }
    \hfill
    \parbox{.45\linewidth}{
    \centering
    \caption{Candidate dataset $\mathcal{S}_2$}
    \label{table:candidate_dataset2}
    \begin{tabular}{ | r | c | c| c | c| }
        \hline
        Id & PTT & PT & Plt. & Dead? \\ \hline \hline
        1 & 0 & 0 & 1 & 1 \\
        2 & 0 & 1 & 1 & 1 \\
        3 & 0 & 0 & 1 & 1 \\
        4 & 1 & 1 & 0 & 0 \\
        \hline
    \end{tabular}
    }
\end{table}

\begin{figure}[h]
    \centering
    \includegraphics[scale=0.30]{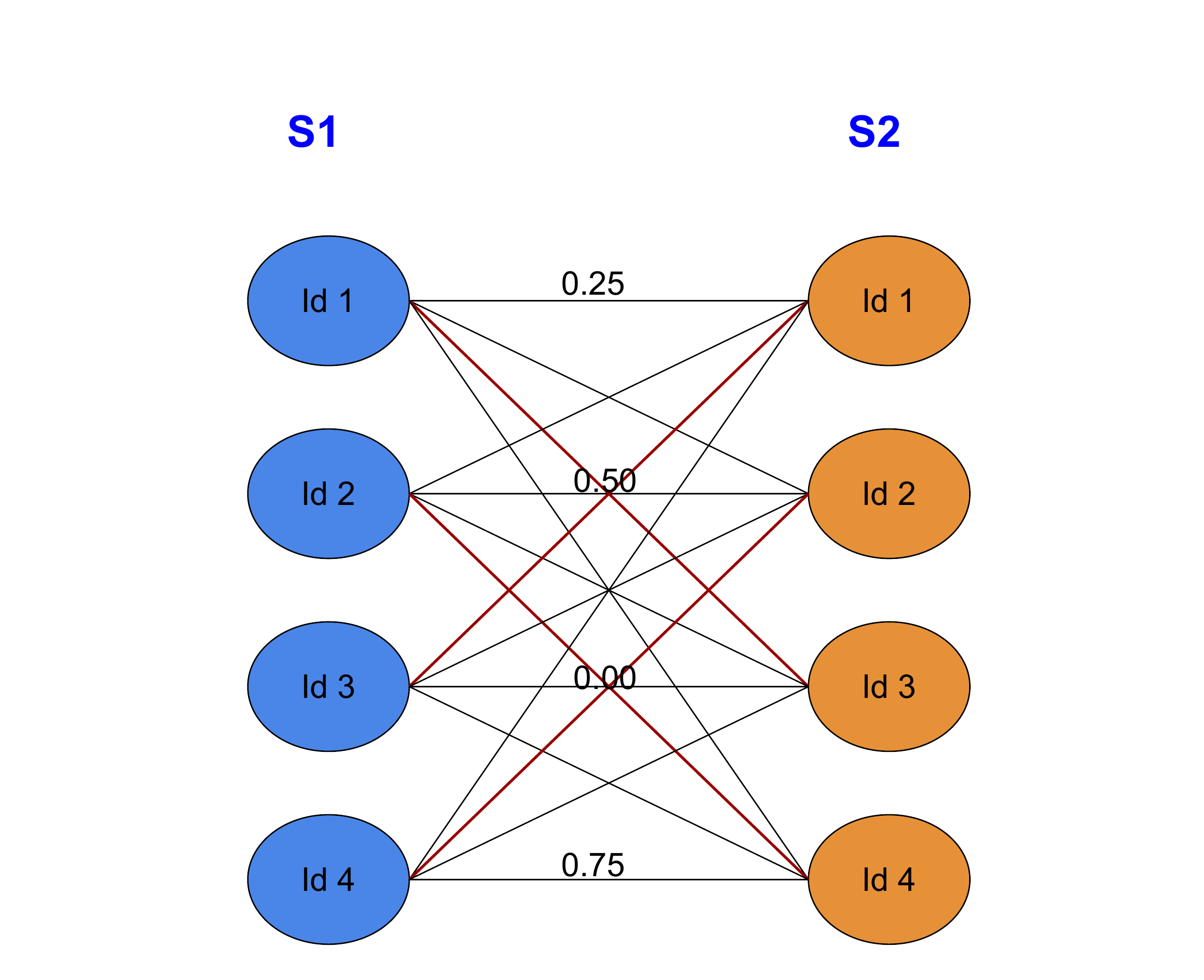}
    \caption{The bipartite graph mapping rows of $\mathcal{S}_1$ to rows of
    $\mathcal{S}_2$}
    \label{figure:bipartite}
\end{figure}

A na\"{\i}ve mapping of rows of $\mathcal{S}_1$ to $\mathcal{S}_2$: row 1 to
row 1, row 2 to row 2 etc will result in the average distance of $0.375$.
However the alternate mapping of row 1 $\rightarrow$ row 3, row 2 $\rightarrow$ row 4,
row 3 $\rightarrow$ row 1, row 4 $\rightarrow$ row 2 will cause the average distance of $0.125$ because row 3 of
$\mathcal{S}_1$ is identical to row 1 of $\mathcal{S}_2$, and row 4 of
$\mathcal{S}_1$ is the same as row 2 of $\mathcal{S}_2$ while the other two
row pairs differ in the value of only one attribute each.

We model this problem of figuring out the optimal mapping from rows of
$\mathcal{S}_i$ to rows of $\mathcal{S}_j$ that gives rise the minimal
average distance as a matching problem in a bipartite graph where one set of
vertices
are the rows of $\mathcal{S}_i$ and the rows of $\mathcal{S}_j$ form the
second set of vertices.
The edges connect rows of $\mathcal{S}_i$ to rows $\mathcal{S}_j$ and
the distance between the two respective rows is the edge weight.
The objective is to find a matching  that minimizes the sum of weights of
edges, and thus minimizes the average distance between the two datasets.
Figure \ref{figure:bipartite} shows the bipartite graph formed for the two
datasets presented in Tables \ref{table:candidate_dataset1}
and \ref{table:candidate_dataset2}.
The optimal matching between vertices is shown in the figure using red lines.

The Hungarian method \cite{kuhn1955hungarian} may be utilized to solve the
matching problem in polynomial time - polynomial in the number of vertices.
The time complexity for the solution is of the order $O(V^3)$ where $V$ is
the vertex set.
In the present work, we implement is a heuristic that runs in linear
time which we find satisfactory and is more efficient:
We rank order rows of $\mathcal{S}_i$, and $\mathcal{S}_j$ based on the sum of
attribute values of respective rows.
The $l^{th}$ ranked row of $\mathcal{S}_i$ is mapped to the $l^{th}$ ranked row
of $\mathcal{S}_j$.
The rationale being that two rows that are similar to each other would have
the sum of their attribute values close to each other, and therefore would be
ranked similarly in the respective datasets $\mathcal{S}_i$, and
$\mathcal{S}_j$.
%we process one row  of $\mathcal{S}_1$ at a time, and we pair it with the yet
%unpaired row of $\mathcal{S}_2$ which is most similar to it.

\subsubsection{Heterogeneous candidate datasets}
While generating $n$ candidate datasets, we ensure that any two candidate
datasets have the average distance between them greater than a threshold
value -- $\delta$ we set.
An important consequence of this stipulation is that we can guarantee that
the rows of the original dataset are reconstructed in the candidate datasets.
Specifically, we provide the following guarantee.

\begin{theorem}
As we increase $n$, the probability $p(n)$ of each row of the original dataset
$\mathcal{H}$ appearing in at least one of the candidate datasets
$\mathcal{S}_1, \mathcal{S}_2, \dots \mathcal{S}_n$ increases. Further,

$$\lim_{n \to \infty} p(n) = 1$$

\end{theorem}

\begin{proof}
    The above assertion follows from two observations: 1) The number of
    ways of reconstructing the original dataset $\mathcal{H}$ is finite.
    2) Any two candidate datasets we generate are separated by $\delta$.
    Therefore, as we generate more and more candidate datasets, the
    likelihood of a row in $\mathcal{H}$ appearing in a candidate dataset
    increases.

    If $\mathcal{H}$ has $m$ rows, and $\mathcal{F}$ is the attribute set,
    then we have a total of $m \times \mathcal{F}$ cells to populate.
    If an attribute $l$ can assume $k_l$ different possible values, then the
    total number of combinations we can create is bounded by:

    $$ m \times \prod_{l \in \mathcal{F}} k_l $$

    The number of combinations in practice will be much smaller because of
    constraints derived from odds ratios, fraction ratios in Algorithm
    \ref{algo:reconstruction}.
    We note that the reconstructed dataset $\mathcal{S}$ is invariant with
    respect to row permutations as we compute the distance between two
    datasets using the bipartite matching formulation.
    Hence, in the above upper bound, we have not considered the combinations
    that can result from permutations of rows.

    Any two datasets that are separated by even the smallest of $\delta$
    should be giving rise to one of the distinct plausible combinations.
    Since the number of combinations is finite and upper bounded by the above
    expression, it must be the case that one of the$\mathcal{S}$s must
    be identical to $\mathcal{H}$.
    As a consequence, the theorem statement that each row of $\mathcal{H}$ will
    be present in one of $\mathcal{S}$s is trivially true.
    Additionally, for this weaker requirement (compared to $\mathcal{H} =
S_j$ for some $j$) to hold, we would need a much smaller number of candidate
    datasets.
\end{proof}

\begin{figure}[ht]
    \centering
    \includegraphics[scale=0.38]{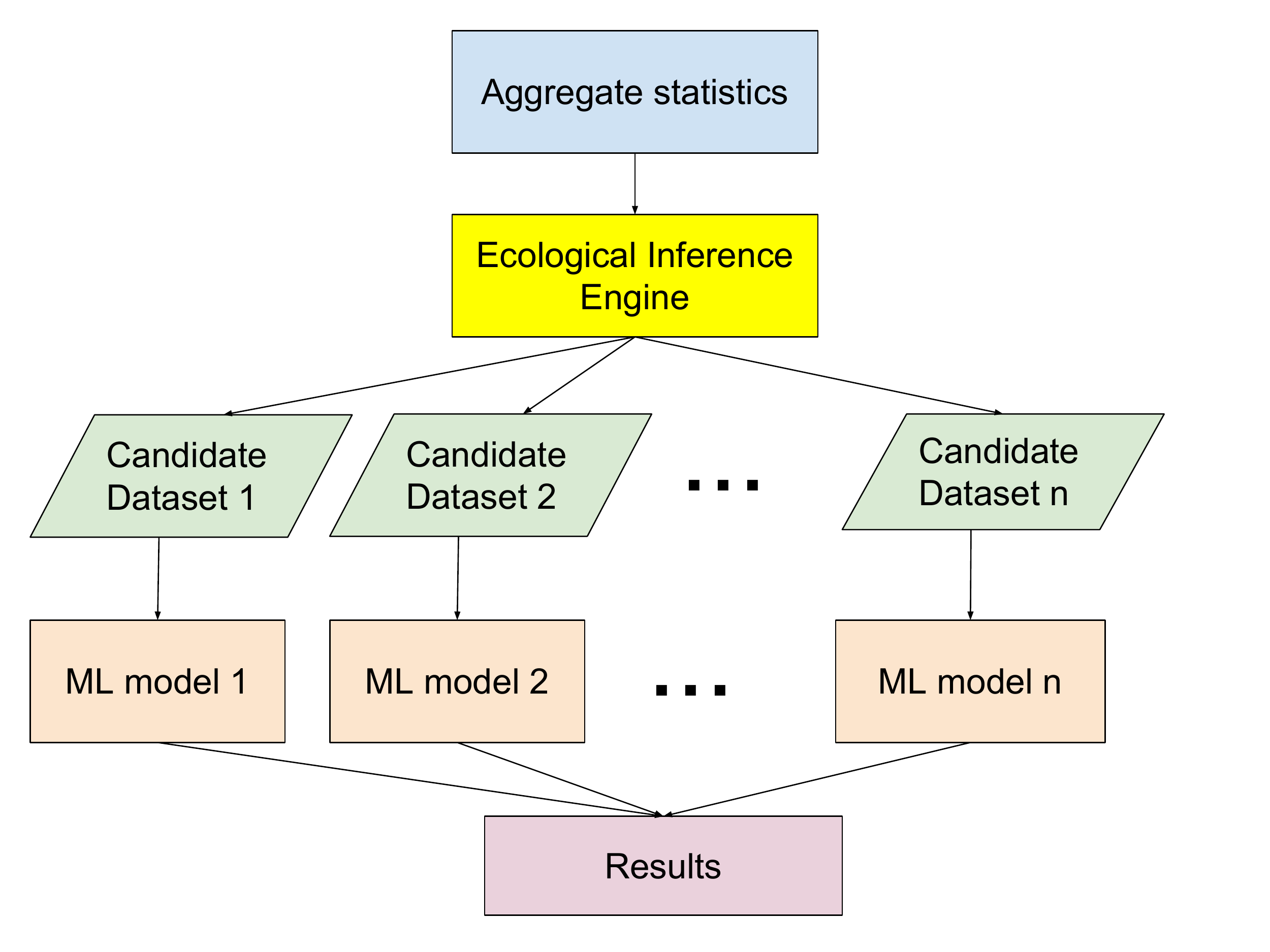}
    \caption{Architecture of the data processing pipeline}
    \label{figure:architecture}
\end{figure}

\subsection{Machine Learning}
\label{section:machineLearning}

Often ecological inference is carried out so that the individual level data
generated may be used for further analysis.
In other words, the end-goal many a time is to perform data analytics on the
ecologically
inferred data rather than use the inferred data directly.
To address this need, we develop a parallel data pipeline architecture to
process aggregate data, perform ecological inference, and use the generated
data to train machine learning (ML) models.

Figure \ref{figure:architecture} presents the architecture of the analytics
system.
The aggregate statistics are inputted into the ecological inference engine
which computes plausible candidate datasets in parallel.
The candidate datasets are then used to train the machine learning models
simultaneously to answer questions of interest.
The outputs of multiple machine learning (ML) models are combined to produce one
aggregate output.
If we are developing the system for a classification problem, then the
majority vote ($mode$) of the $n$ ML models will be the combined result of
the system.
If it is a regression problem that the system is
addressing then the $arithmetic ~mean$ of the outputs of the different models
will be
the
aggregate outcome of the models.

\begin{table*}[h]
    \centering
    \caption{Parameters for ground truth generation.
    Legend: Plate - Platelet, DOA ratio - Dead or Alive ratio}
    \label{table:feature_parameters}
    \begin{center}
        \begin{tabular}{ | l | l | l| l | l | l | l | l | l | l | }
            \hline
            Config. & Gender OR & Gender ratio & PT
            OR & PT ratio &
            PTT OR & PTT ratio &  Plate OR & Plate ratio & DOA ratio\\ \hline
            \hline
            1 & 2 & 0.6 & 4 & 0.3 & 6 & 0.2 & 8 & 0.1 & 0.1 \\ \hline
            2 & 6 & 0.7 & 2 & 0.2 & 4 & 0.1 & 6 & 0.3 & 0.1 \\ \hline
            3 & 8 & 0.8 & 8 & 0.1 & 2 & 0.3 & 4 & 0.4 & 0.2 \\ \hline
            4 & 10 & 0.9 & 8 & 0.2 & 4 & 0.4 & 2 & 0.5 & 0.2 \\ \hline
            5 & 10 & 0.6 & 4 & 0.1 & 2 & 0.2 & 6 & 0.5 & 0.3 \\ \hline
            6 & 4 & 0.5 & 2 & 0.3 & 6 & 0.1 & 8 & 0.2 & 0.3 \\ \hline
            7 & 2 & 0.5 & 6 & 0.3 & 8 & 0.2 & 10 & 0.1 & 0.4 \\ \hline
            8 & 6 & 0.5 & 4 & 0.2 & 8 & 0.1 & 10 & 0.5 & 0.4 \\ \hline
            9 & 8 & 0.5 & 2 & 0.4 & 10 & 0.3 & 10 & 0.4 & 0.45 \\ \hline
            10 & 10 & 0.5 & 6 & 0.4 & 10 & 0.3 & 10 & 0.3 & 0.49 \\ \hline
        \end{tabular}
    \end{center}
\end{table*}

%We will be drawing conclusions on the reliability of the predictions of the ML
%models by observing how the outcomes of the $n$ ML models differ:
%If they hardly differ, that is a sign that the candidate datasets are
%``similar'' from the point of view of their ability to predict the outcome of
%interest.
%Therefore, we can infer high confidence in the analysis results.
%On the other hand, if the outcomes of the different ML models are widely
%varying, that suggests that very dissimilar datasets can give rise to
%the same aggregate measures, and are causing the downstream divergence in the
%ML model predictions.
%Consequently, we attach a low confidence to the results.
%However, by choosing a sufficiently large $n$, we can steer the analytics
%engine to generate more probable underlying datasets in greater numbers,
%thereby enhancing the confidence in the results.

\section{Case Study: Predicting Acute Traumatic Coagulopathy Outcomes}
\label{section:case_study}
We use the developed parallel data pipeline on the Acute Traumatic
Coagulopathy (ATC) data to predict mortality.
The system is implemented in the \emph{Apache Spark} \cite{apache_spark,Zaharia:2016:ASU:3013530.2934664}
version 2.2.0  cluster
computing framework.
Spark is an in-memory Big data computing framework which helps us scale our
implementation to large volumes of data seamlessly.
%That is to predict whether a given patient will be dead, or alive.
The predictor variables and the possible values they can assume are the
following:
\begin{itemize}
    \item Gender: Male or Female
    \item Platelet count: Abnormal, or Normal
    \item Prothrombin Time (PT): Abnormal or Normal
    \item Partial Thromboplastin Time (PTT): Abnormal or Normal
    \item Age: age of the patient
\end{itemize}

\subsection{Experimental Set Up}
%\FloatBarrier
We perform two sets of experiments:

1) We first assess the efficacy of ecological
inference algorithm in reconstructing the datasets that are similar to the
ground truth data.
We feed the aggregate data to the ecological inference algorithm, and
generate candidate datasets.
Then we compute the similarity scores between the ground truth data and
candidate datasets.

2) We run the entire the data analytics pipeline to understand its effectiveness.
For this, we use the generated data to train the Random Forest \cite{liaw2002classification}
machine learning models.
The performance of the Random Forests are assessed against the ground truth.

\paragraph{\textbf{Ground truth data for evaluation}}
The ground truth data are
synthetically created.
We generate a number of datasets that differ in their characteristics.
For each predictor, we vary odds ratios from 2 to 10, and occurrence ratios
from 0.1 to 0.5.
For the outcome variable -- mortality, the class ratios range from 0.1 to
0.5 as well.

We note that in the ATC context, a predictor variable, for example, either
can have a normal value or an abnormal value.
Similarly, the outcome variable -- mortality can set to ``dead'' or ``alive''.
Thus, in binary valued variables scenario such as this, the experimental
results when we set occurrence ratio of a variable to say, 0.6 will be
identical to the case when the occurrence ratio is 0.4 ($1 - 0.6 = 0.4$).
The results when occurrence ratio is 0.7 will be indistinguishable to the
case when the ratio is 0.3 ($1 - 0.7 = 0.3$).
Therefore, we vary the ratios only between 0.1 to 0.5.

A total of 10 datasets are synthesized in this manner using 10 different
sets of parameters.
Table \ref{table:feature_parameters} enumerates the 10 sets of
parameters used.
Each set of parameters consists of a unique combination of odds ratios and
occurrence ratios for predictors, and outcome class ratios.
For example, the first configuration in the Table has Gender ratio of 0.6
indicating that the male fraction of the population is 0.6 while the OR of
 dying is 2 if one is a male patient.
The ORs of dying because of abnormal PT, PTT, and Platelet values are 4, 6,
and 8 respectively.
The fractions of patients who have abnormal PT, PTT, and Platelet counts are 0.3,
0.2, and 0.1 in that order.
The DOA ratio is 0.1 signifying that 0.1 fraction of the patients died.
Similar interpretations are ascribed to the other $9$ parameter configurations.
For all configurations, the age values are generated such that the mean age
of patients is 36, and standard deviation is 19.
%We employ the reconstruction algorithm -- Algorithm \ref{algo:reconstruction}
%itself to generate the ground truth data also: we input the desired
%properties that we would like to see in the generated dataset, and the
%algorithm outputs a dataset with precisely those characteristics.
%In order to make sure any of the candidate datasets is not the same as the
%ground truth dataset as we are using the same algorithm to generate both, we
%input different seed values for the
%random number generator in the separate invocations of the algorithm.
%Furthermore, for each parameter combination we generate $10$ ground truth
%datasets.
Each ground truth dataset consists of $10,000$ patient records.

\paragraph{\textbf{Performance of aggregate data analytics pipeline}}
Once we have the ground truth data created in this manner, for each parameter
configuration shown in Table \ref{table:feature_parameters}, we run the
pipeline shown in Figure \ref{figure:architecture}.
For each ground truth dataset, we output $9$ candidate datasets and compute
the similarity scores between the ground truth data and candidate datasets.
The Random Forest machine learning models are trained using the $9$ candidate
datasets.
A Random Forest comprises of 50 decision trees, each of depth 8.

We record the performance
of Random Forests in terms of \emph{accuracy}, \emph{precision}, and
\emph{recall}.
The ground truth data sans the label is inputted to the Random Forest
model to predict mortality.
Each patient can have the label as either \emph{dead}, or \emph{alive} and
\emph{dead} is the positive class for our evaluation.
Consequently, \emph{accuracy}, \emph{precision}, and \emph{recall} are
defined as follows.

\begin{align*}
    \text{accuracy} &= \frac{\text{\#patients correctly classified
    as either dead or alive}}{\text{total number of patients}} \\
    \text{precision} &= \frac{\text{\#patients correctly classified
    dead}}{\text{\#patients classified dead}} \\
    \text{recall} &= \frac{\text{\#patients correctly classified
    dead}}{\text{\#actual patients dead}}
\end{align*}

%For each parameter configuration, there are $10$ different ground truth datasets
%available.
%We run the pipeline for each ground truth dataset.
%The different stages of the pipeline are:
%1) From a ground truth dataset the  aggregate statistics are computed
%2) The aggregate  statistics are used to generate $9$ separate candidate
%datasets in parallel that are the estimations of the ground truth.
%3) The Random Forest models are trained using the $9$ candidate datasets in
%parallel.
%The Random Forests comprise of 50 decision trees, and each tree is of depth 8.
%4) The ground truth data sans the label is inputted to the Random Forest
%model to obtain labels.
%5) The Random Forest labels are compared with the actual labels in the ground
%truth data and the following metrics are generated: \emph{accuracy},
%\emph{precision}, and \emph{recall}.

\subsection{Experimental Results}
\label{section:experiments}

\paragraph{\textbf{Similarity between the ground truth and reconstructed
datasets}}

Figure \ref{figure:similarity_score} shows the similarity score between the
ground truth dataset and the reconstructed datasets for various parameter
configurations shown in Table \ref{table:feature_parameters}.
For each configuration we generate a ground truth dataset meeting the
parameter specifications.
We calculate aggregate statistics on the ground truth dataset and use that as
the input to the ecological inference algorithm.
The algorithm outputs 9 different candidate datasets.
The threshold distance between any pair of candidate datasets is set to $0.15$.
The figure shows the average similarity score between the ground truth and
all the
candidate datasets (the highest possible score being 1).
We observe that the similarity score is consistently high and hovers around
$0.80$ for all configurations.
The variance in similarity scores between a ground truth data and the
candidate datasets generated for it is low.
For Config.1, for example, the minimum similarity score is $0.8192$, while
the maximum is $0.8230$.
The average score is $0.8202$.

\begin{figure}[h]
    \centering
    \includegraphics[scale=0.42]{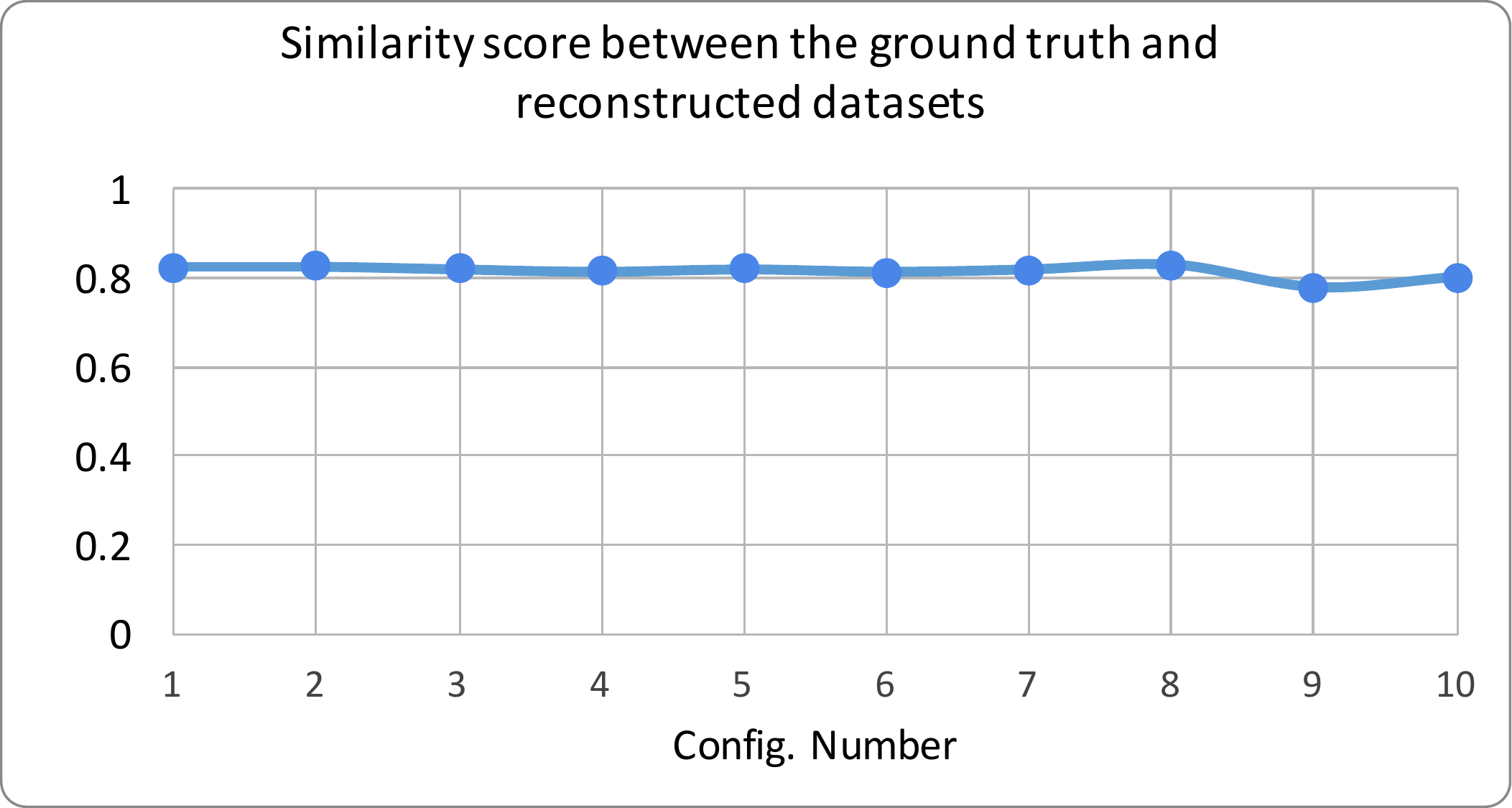}
    \caption{The average similarity score between the ground truth and
    reconstructed datasets}
    \label{figure:similarity_score}
\end{figure}

\begin{figure}[h]
    \centering
    \includegraphics[scale=0.42]{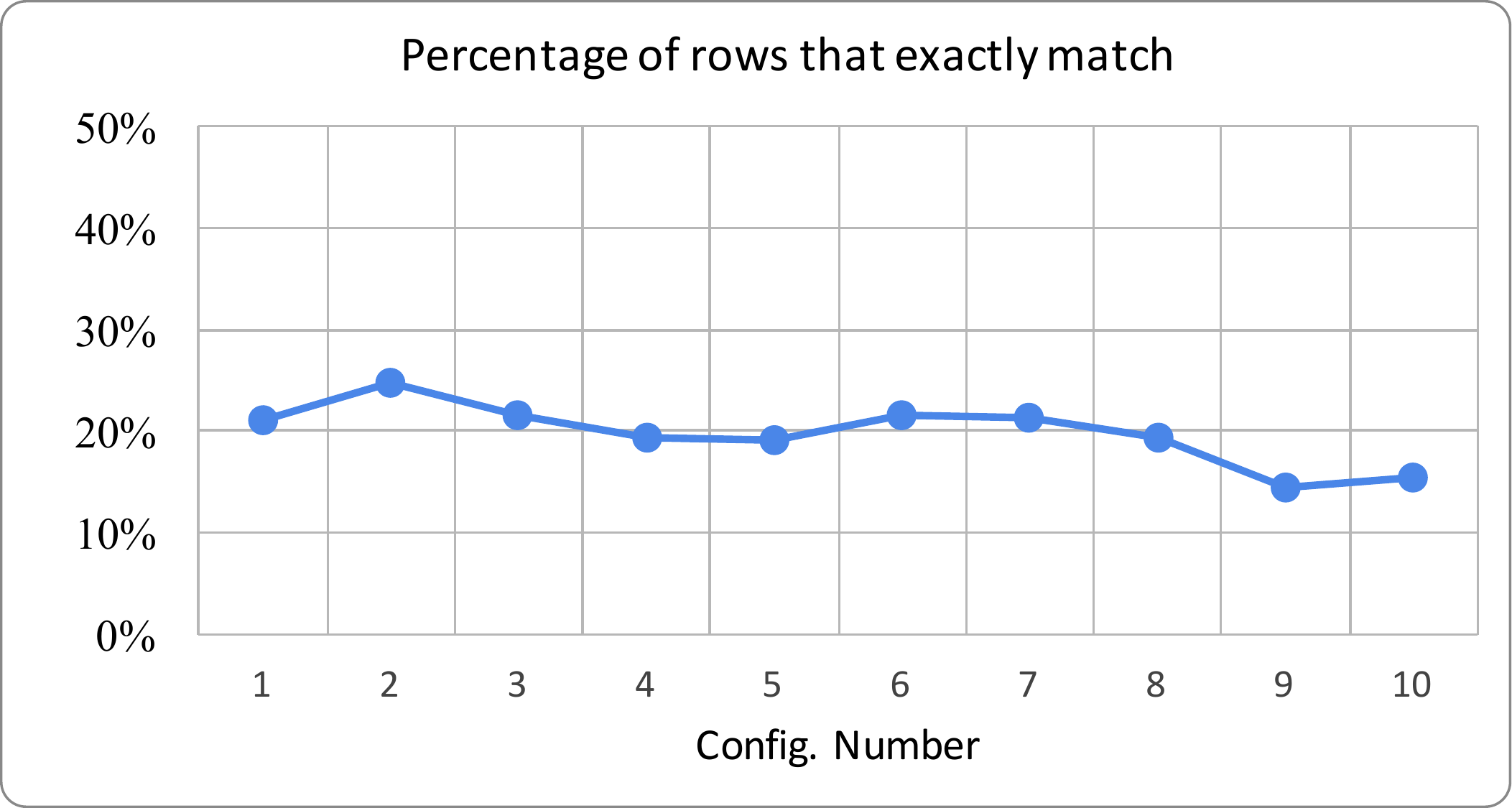}
    \caption{The average percentage of rows that match exactly between the
    ground truth and a reconstructed dataset}
    \label{figure:zero_dist_percent}
\end{figure}

We additionally record the percentage of rows of the ground truth dataset and
candidate datasets that match exactly.
Figure \ref{figure:zero_dist_percent} plots the average percentage over
different candidate datasets.
Between 10\% and 30\% of the rows match exactly.

We note that the similarity scores and exact match rows are computed using the
optimized greedy
algorithm (\S \ref{section:similarity_score}), and therefore the
computed scores and row counts represent lower bound values.
The actual similarity scores and row counts could be higher if  the
Hungarian
method to
derive matching in the bipartite graph is used.
It is  expensive computationally, however.

\paragraph{\textbf{Overall prediction performance}}
Figure \ref{figure:config_performance} shows the performance of the parallel
aggregate data analytics pipeline in terms of accuracy, precision, recall
achieved for the different parameter configurations listed in Table
\ref{table:feature_parameters}.
%For each configuration, as noted earlier, 10 different ground truth datasets
%are generated.
%The performance metrics depicted in the Figure are the \emph{arithmetic
%mean} of
%performance of the Random Forest models on those 10 datasets.
%The performance achieved by the Random Forest model on the 10 datasets is
%similar for all configurations.
%The variance in accuracy, precision, and recall numbers is very low, and
%there are no outliers either.
%For example, for Config 1, the variance for accuracy is $1.78 \times 10^{-7}$,
%while variance for precision and recall are $3.21 \times 10^{-5}$, and $3.65
%\times 10^{-5}$ respectively.
%This indicates that the reconstruction algorithm, and the analytics pipeline are
%robust in the sense that the conclusions and results drawn from the analysis
%are
%reliable.

\begin{figure}[h]
    \centering
    \includegraphics[scale=0.42]{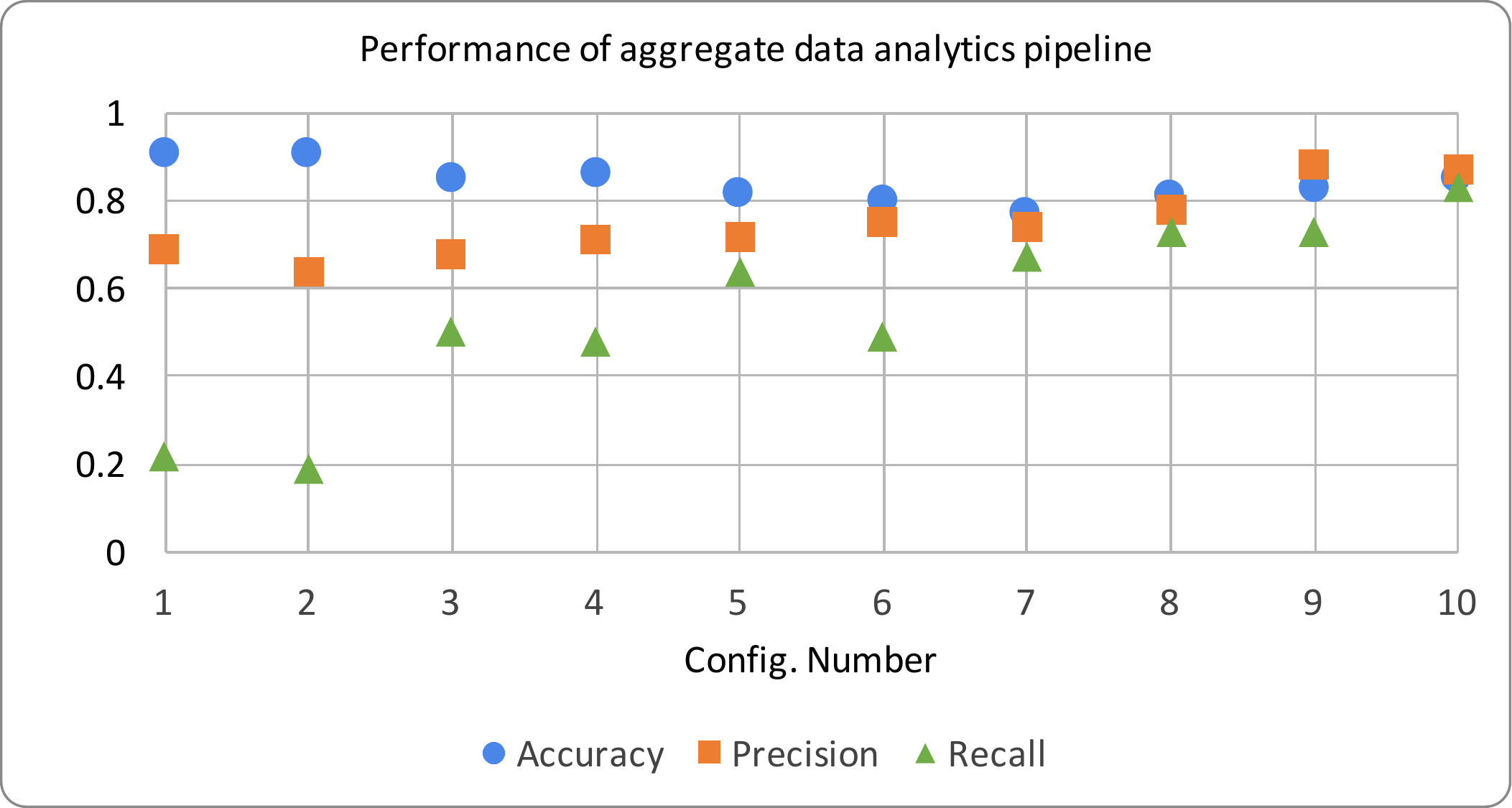}
    \caption{Performance of the pipeline for various configuration parameters}
    \label{figure:config_performance}
\end{figure}

We observe that high \emph{accuracy} is obtained for all configurations.
It fluctuates between 0.77 and 0.91.
Config. 7 has the lowest accuracy -- 0.77 while Config. 1 and 2 both attain
accuracy of 0.91.
The \emph{precision} metric which is the ratio of true positives (the number of
patients that the model correctly classified as dead) to the sum of true
positives and false positives (the
number of patients the model correctly or incorrectly classified as dead),
improves as we move from Config. 1 to 10.
The reason being that the DOA ratio (the fraction of patients who died)
increases with the configuration number.
The result is that the model has more positive (dead patients) samples to
learn from, and its precision increases.
A similar influence is at play with respect to the \emph{recall} metric as
well -- as the DOA ratio increases so does the recall.
As the fraction of dead patients rises, the model is able to correctly
identify a larger fraction of dead patients.
The recall rate is 0.22 when only 10\% of patients are dead (Config. 1), and
it climbs to 0.50 when the percentage of dead patients rises moderately to
20\% (Config. 3).
The recall is highest at 0.83 when 49\% of patients have died (Config. 10).

In sum, we note consistently high accuracy and precision measures while
recall is high when the DOA ratio is moderate to high, but is on the lower
side when the DOA ratio is small.

\begin{figure}[h]
    \centering
    \includegraphics[scale=0.42]{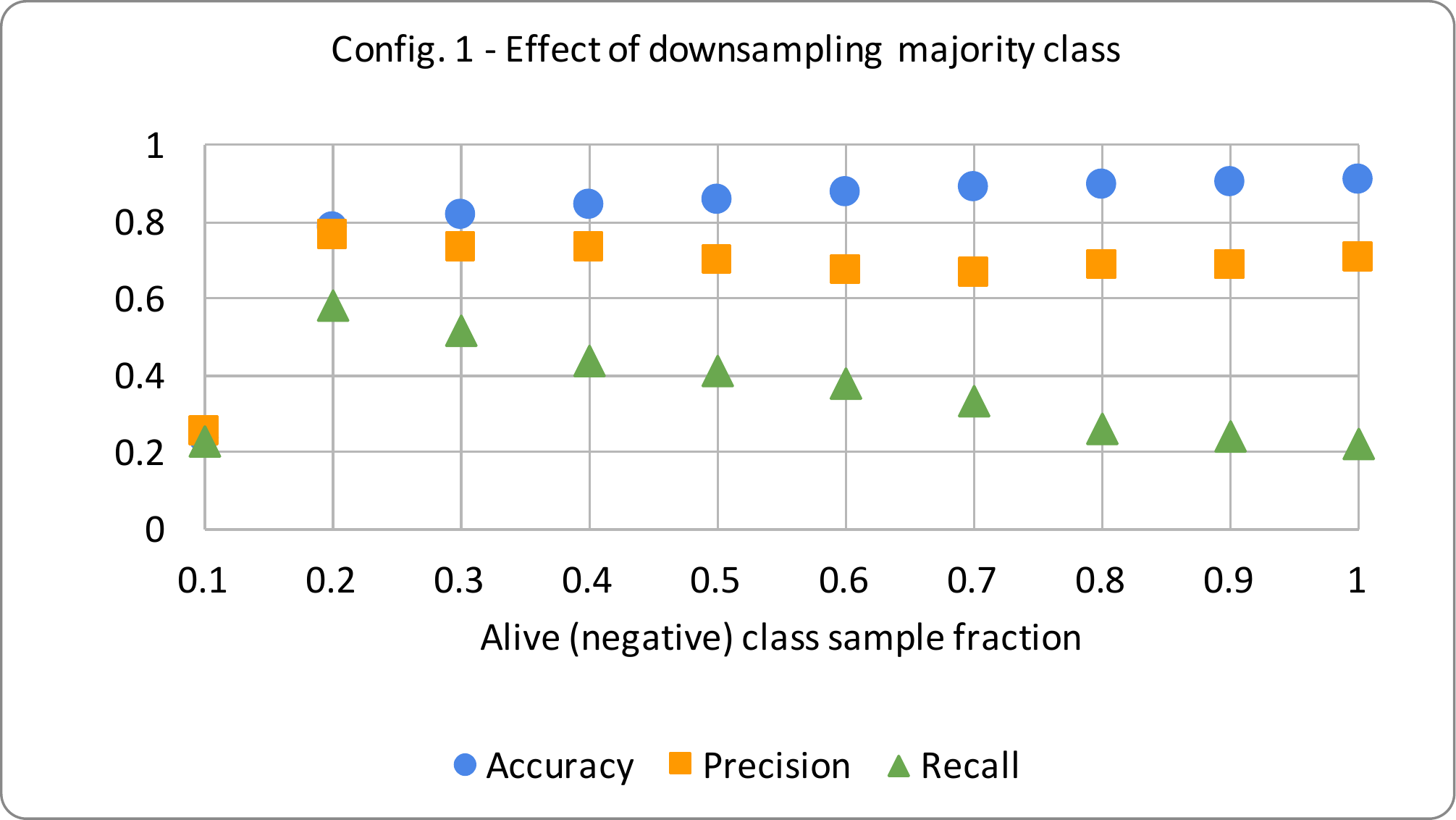}
    \caption{Config. 1: Effect of downsampling majority class in improving
    recall}
    \label{figure:config1_downsampling}
\end{figure}

\paragraph{\textbf{Boosting recall via undersampling of majority class}}
For a context like Acute Traumatic Coagulopathy which is the subject of case
study in this work, having a high \emph{recall} rate at all mortality levels
is crucial: a false negative may have a more adverse impact than a false
positive.
Here, the main reason why the recall rate is low for Config. 1, and 2 is that
the DOA ratio is small, namely 0.1.
As a result, the Random Forest model is being trained with
severely imbalanced candidate datasets -- 90\% of the patients have class label
`Alive', and a small percentage of patients -- 10\% have class label `Dead'.
To mitigate the imbalanced nature of the dataset, we undersample the `Alive'
patients and investigate if that improves the recall rate.
Indeed, as the dataset becomes more balanced, the recall value rises.
Simultaneously, the accuracy very slightly degrades, and there is no
discernible impact on precision.

\begin{figure}[h]
    \centering
    \includegraphics[scale=0.42]{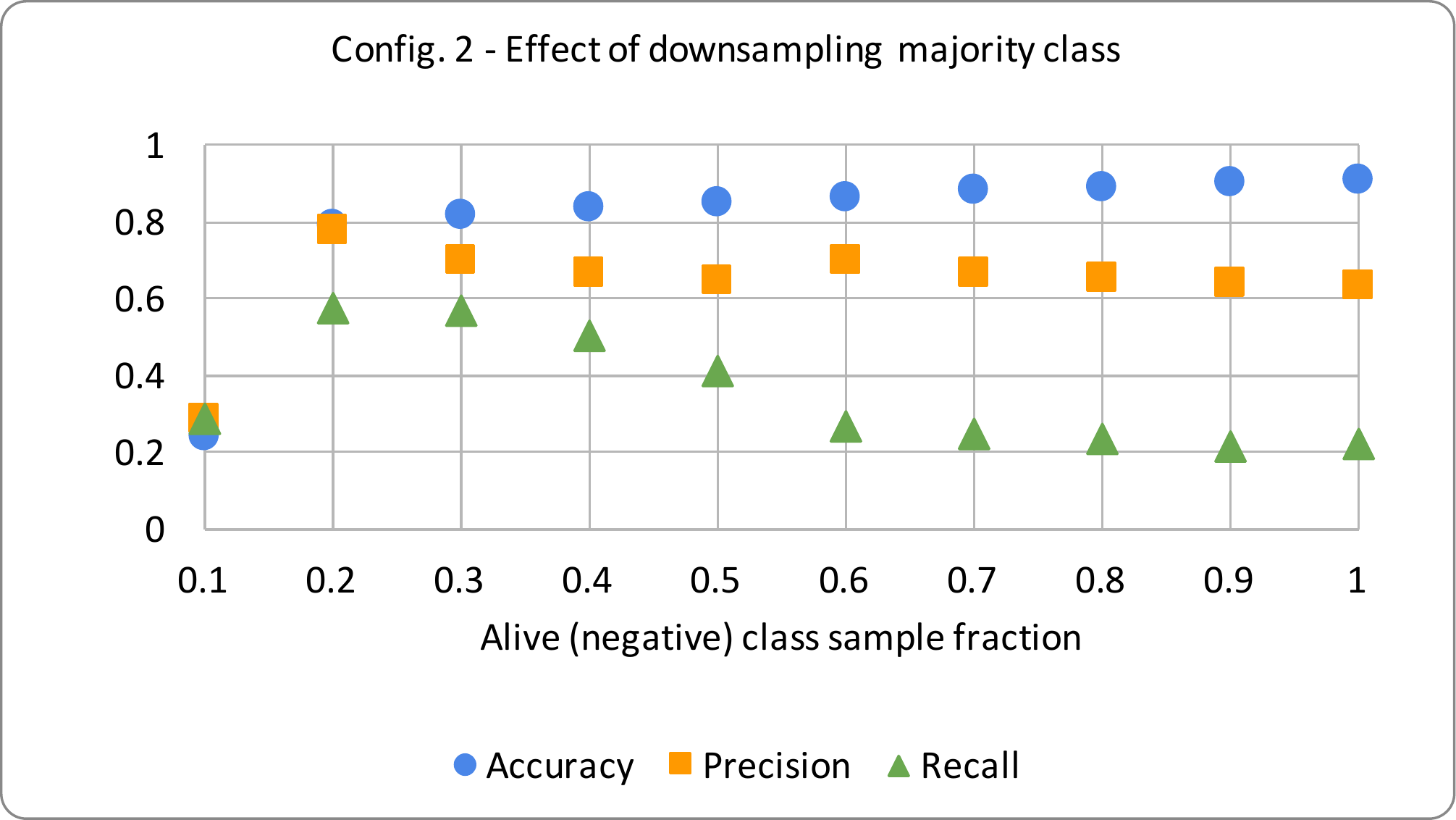}
    \caption{Config. 2: Effect of downsampling majority class in improving
    recall}
    \label{figure:config2_downsampling}
\end{figure}

Figures \ref{figure:config1_downsampling} and
\ref{figure:config2_downsampling} show the outcomes of these experiments.
Figure \ref{figure:config1_downsampling} outlines trends seen for accuracy,
precision, and recall measures as the sampling rate for `Alive' class is
varied from $0.1$ to $1$.
When the alive class sample fraction is $0.1$, the Random Forest model is
being trained with records of all dead patients which is $10\%$ of the entire
candidate dataset, and $0.1$ fraction of the alive patients' records which
amounts to $9\%$ ($0.1$ fraction of $90\%$ of records) of the entire candidate
dataset.
When the Alive class sample rate is $0.2$, the model is being trained with
$10\%$ of the positive samples (patients who are dead), and $18\%$ of the
negative samples (patients who are alive).
The highest recall of $0.58$ is achieved when the majority class sample
fraction is $0.2$.
When the majority sample fraction increases further, recall rate comes down
while accuracy improves.
Accuracy when the `Alive' sample fraction is $0.2$ is $0.79$, and it reaches
a high of $0.91$ when the entire dataset is used for training.
Nearly identical trends are seen in Figure \ref{figure:config2_downsampling}
which shows the effect of undersampling of majority class for Config. 2
parameters.

\paragraph{\textbf{Controlled experiments}}
\begin{figure}[h]
    \centering
    \includegraphics[scale=0.42]{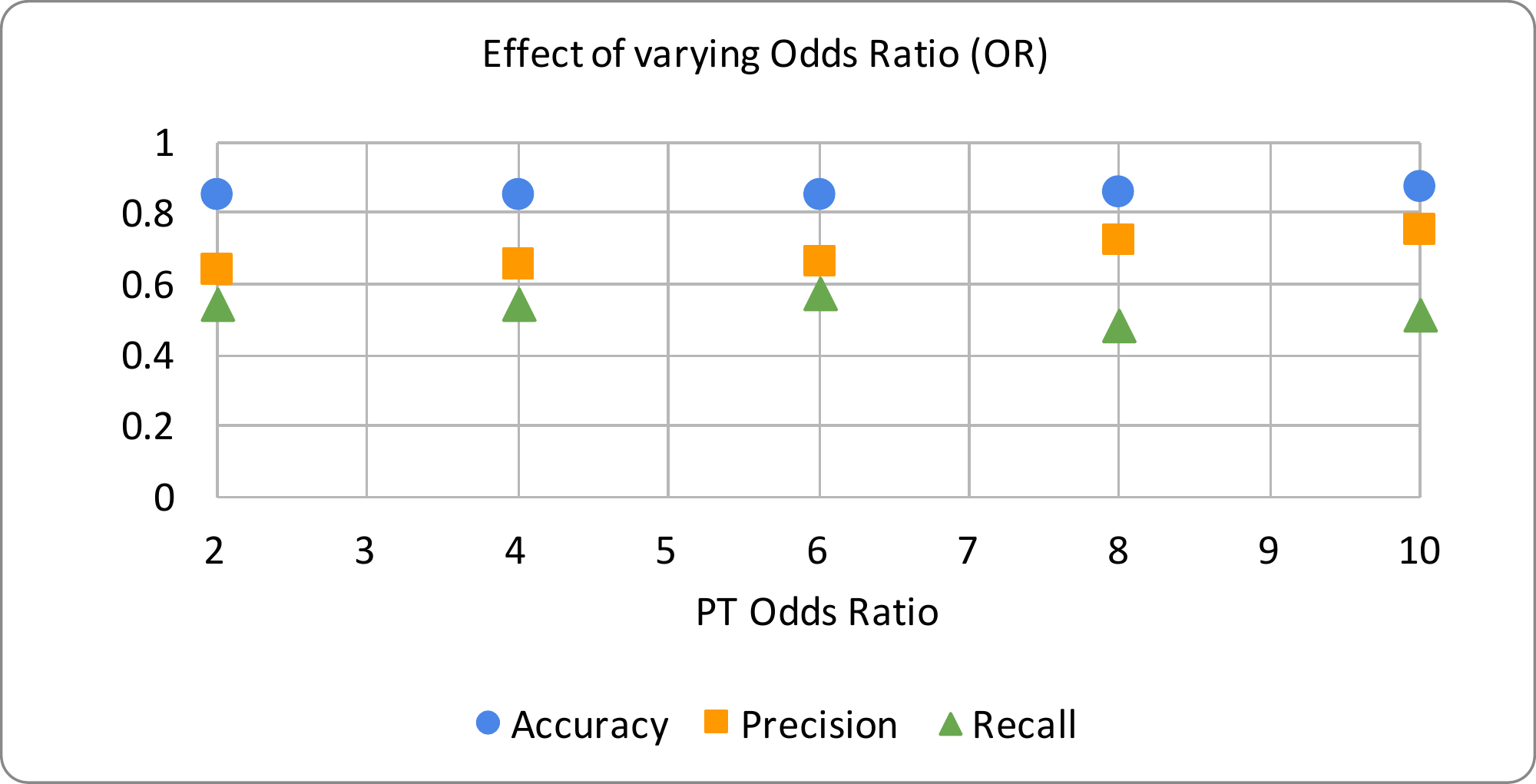}
    \caption{Controlled experiments: effect of varying PT OR alone}
    \label{figure:controlled_OR}
\end{figure}

Figure \ref{figure:config_performance} depicts the performance for the
parameter combinations shown in Table \ref{table:feature_parameters}.
Although we were able to, to a great extent, discern how each of the
parameters affects accuracy, precision, and recall achieved by the Random
Forest models, we perform ``controlled'' experiments to tease apart the
effects of each of the parameters further.
In this set of experiments, we individually vary 1) Odds Ratio 2) DOA ratio, and
3) Occurrence ratio of a predictor variable while keeping the rest of the
parameters
constant.

The predictive performance of PT OR is shown in Figure
\ref{figure:controlled_OR}.
The PT OR is varied from 2 to 10.
We notice that accuracy, and recall are nearly invariant with respect to OR.
The precision measure however sees an uptick as the OR is increased.
This is because, a higher OR indicates a stronger connection between an
abnormal PT value and mortality.
Therefore, the model is able to bring down false negatives thereby increasing
precision.

\begin{figure}[h]
    \centering
    \includegraphics[scale=0.42]{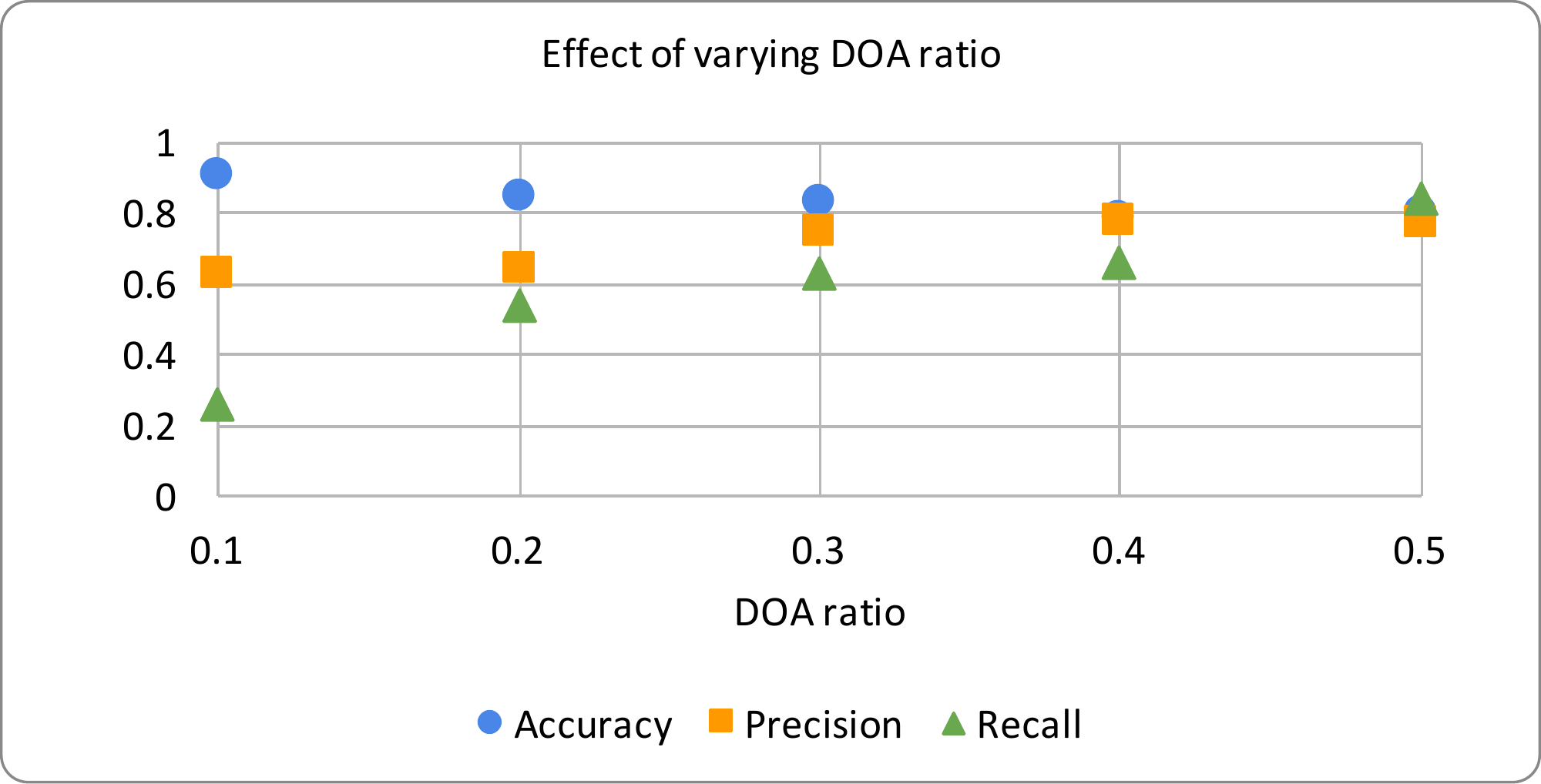}
    \caption{Controlled experiments: effect of varying DOA ratio alone}
    \label{figure:controlled_DOA}
\end{figure}

Figure \ref{figure:controlled_DOA} illustrates the impact DOA ratio has on
accuracy, precision, and recall.
DOA ratio is varied from 0.1 to 0.5 (since this is a binary classification
problem, varying DOA ratio from 0.6 to 0.9 will cause mirror image
performance: performance at DOA ratio $x$ will be the same as when it is $1 - x$).
We notice that as the DOA ratio is increased, recall rate improves because
the dataset becomes more balanced -- the number of records with the positive
class labels will approach the number of records with the negative class labels.
Precision measure also shows an upward trend as the DOA ratio is increased.
Accuracy slightly dips as the DOA ratio becomes larger.
This is explainable from the fact that when the DOA ratio is small, even if
the model predicts everyone as `Alive', the accuracy will be high.
But, as the DOA ratio becomes greater, the model has to learn to be more
discriminating, and consequently accuracy takes a small hit.

\begin{figure}[h]
    \centering
    \includegraphics[scale=0.42]{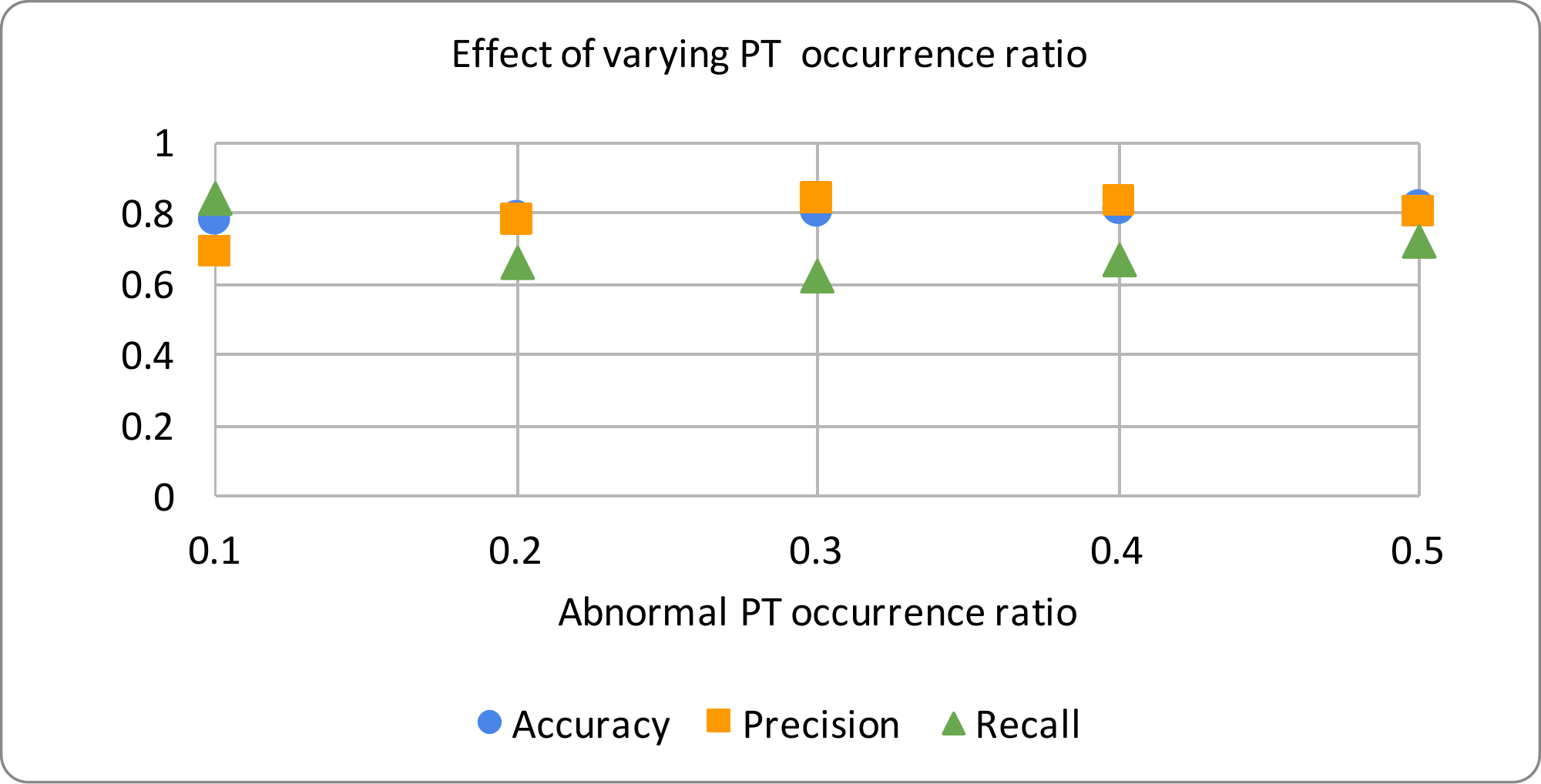}
    \caption{Controlled experiments: effect of varying PT occurrence ratio
    alone}
    \label{figure:controlled_PT_ratio}
\end{figure}

In Figure \ref{figure:controlled_PT_ratio}, we see how increasing the
fraction of patients that have an abnormal PT value affects performance.
The abnormal PT occurrence ratio is varied from 0.1 to 0.5, and accuracy,
recall are little affected by its variation.
The precision on the other hand sees a small gain as the occurrence ratio is
raised.

%\FloatBarrier

\section{Related Work}
\label{section:related}

Colbaugh and Glass \cite{colbaugh2017learning} present a method to build machine learning
models to predict individual-level labels using aggregate data.
The methodology consists of three steps: 1) feature extraction, 2)
aggregate-based prediction, 3) prediction refinement.
This approach is similar to hierarchical classification/regression: first the
coarse label is predicted, and then a fine-grained label is assigned.
Their method assumes that individual-level labels are present to train the
model in the refinement step.
Though the high-level goals of their work and our work are similar, the
problem we
address is distinct in that our
solution approach is applicable even when labels for individual-level data are
not available.
Additionally, we have presented an algorithm that reconstructs individual-level
data from
various aggregate statistics including odds ratios.

Gary King's book \cite{king2013solution} illustrates various known techniques
for reconstructing individual behavior from aggregate data, also termed
ecological inference.
The main approach advocated is to utilize a method of bounds where
in the domain specific knowledge is used to place bounds on variables.
The second principal technique in the context of voting in elections is to
perform ecological inference for
smaller precincts and then combine the results for a larger context, say a
state.
This is best illustrated with regards to voting data.
For example, Imai and King \cite{imai2004did} apply ecological inference for
the 2000 U.S. presidential elections to answer the question: ``Did Illegal
Overseas Absentee Ballots Decide'' the election?
Here, the inference problem is figuring out how many valid, and invalid
absentee ballots were
cast for each of the two presidential candidates given that there were a
total of $680$ invalid ballots, $1,810$ valid ones, and the two candidates
received $836$, and $1,575$ votes overall.
%The techniques presented in this paper incorporate the bounds method implicitly.
%Our solution derives individual level behavior, for example, how each voter
%might have voted.
%This differs from Gary's techniques in that Gary's techniques help answer questions which cannot
%be answered from aggregate data in a straight-forward manner such as the
%one posed about the 2000 U.S. election results.

Musicant and others \cite{musicant2007supervised} address an
interesting problem of performing supervised learning
by training on aggregate outputs.
In their framework, the training set contains observations for which all attribute
values are known, but the output variable's value is known only in the aggregate.
The need for such an analysis arose when studying mass spectrometry data.
They examine how k-nearest neighbor, neural networks, and support vector
machines can be adapted for this problem.

MacLeod et al. \cite{macleod2003early} analyze data collected in a
prospective study on patients admitted to a Level I trauma center.
They apply logistic regression using the various known acute traumatic
coagulopathy predictors, and determine that coagulopathy is strongly linked
to mortality in trauma patients.

\section{Conclusion}
\label{section:conclusion}

In this paper, we described the architecture of a system for performing
analysis of
aggregate statistics.
We developed a novel algorithm to reconstruct individual-level data from
aggregate data.
To increase the confidence in the data analysis performed, we set up the data
reconstruction algorithm to compute a parametric number of individual-level
datasets.
Furthermore, this step is completely parallel and makes the generation of
datasets efficient.
The datasets will be used to train several machine learning models
concurrently to discover knowledge.
We performed extensive experiments to evaluate the predictive performance of
the system in the context of a medical condition called Acute Traumatic
Coagulopathy.
The experimental results indicate that the system developed achieves good
performance thereby showing that the end-to-end aggregate data analytics
system developed can be reliably used to extract knowledge from aggregate data.

\balance
\bibliographystyle{IEEEtran}
{\footnotesize
\bibliography{bibliography}
}

% that's all folks
\end{document}